\documentclass[twoside,11pt]{article}

\usepackage{blindtext}

\usepackage{jmlr2e}

\usepackage{amsmath, amsfonts, amssymb}
\usepackage{subcaption}
\usepackage{booktabs}
\usepackage{bm}
\usepackage{algorithm}
\usepackage{algpseudocode}
\usepackage{multirow}

\usepackage{lastpage}
\jmlrheading{23}{2022}{1-\pageref{LastPage}}{1/21; Revised 5/22}{9/22}{21-0000}{M. Arashi and M. Amintoosi}

\ShortHeadings{Stein-Rule Shrinkage for Stochastic Gradient Estimation in High Dimensions}{Stein-Rule Shrinkage for Stochastic Gradient Estimation in High Dimensions}
\firstpageno{1}

\begin{document}

\title{Stein-Rule Shrinkage for Stochastic Gradient Estimation in High Dimensions}

\author{\name M. Arashi \email arashi@um.ac.ir\\
       \addr Department of Statistics\\
       Faculty of Mathematical Sciences\\
       Ferdowsi University of Mashhad\\
       P. O. Box 1159, Mashhad 91775, Iran
       \AND
       \name M. Amintoosi \email m.amintoosi@um.ac.ir \\
       \addr Division of Computer Science\\
       Department of Applied Mathematics\\
       Faculty of Mathematical Sciences\\
       Ferdowsi University of Mashhad\\
       P. O. Box 1159, Mashhad 91775, Iran}

\editor{My editor}

\maketitle

\begin{abstract}
Stochastic gradient methods are central to large-scale learning, but they treat mini-batch gradients as unbiased estimators, which classical decision theory shows are inadmissible in high dimensions. We formulate gradient computation as a high-dimensional estimation problem and introduce a framework based on Stein-rule shrinkage. We construct a gradient estimator that adaptively contracts noisy mini-batch gradients toward a stable estimator derived from historical momentum. The shrinkage intensity is determined in a data-driven manner using an online estimate of gradient noise variance, leveraging statistics from adaptive optimizers. Under a Gaussian noise model, we show our estimator uniformly dominates the standard stochastic gradient under squared error loss and is minimax-optimal. We incorporate this into the Adam optimizer, yielding SR-Adam, a practical algorithm with negligible computational cost. Empirical evaluations on CIFAR10 and CIFAR100 across multiple levels of input noise show consistent improvements over Adam in the large-batch regime. Ablation studies indicate that gains arise primarily from selectively applying shrinkage to high-dimensional convolutional layers, while indiscriminate shrinkage across all parameters degrades performance. These results illustrate that classical shrinkage principles provide a principled approach to improving stochastic gradient estimation in deep learning.
\end{abstract}



\section{Introduction}

The inadmissibility of unbiased estimators in high-dimensional settings is one of the cornerstones of modern frequentist decision theory. In their seminal works, \cite{stein1956}  and  \cite{james1961} showed that for $p \geq 3$, the usual unbiased estimator of the mean of a multivariate normal distribution is inadmissible under quadratic loss. In particular, they demonstrated that shrinkage estimators exist which uniformly dominate the maximum likelihood estimator in terms of mean squared error over the entire parameter space. These results fundamentally challenge the classical preference for unbiasedness in high-dimensional estimation problems.

Modern deep learning operates squarely within this high-dimensional regime. Contemporary neural networks routinely involve millions, or even billions, of parameters, rendering low-dimensional statistical intuition largely inapplicable. Nevertheless, the dominant optimization paradigm in deep learning --stochastic gradient descent (SGD) and its numerous variants-- relies critically on mini-batch gradients that are treated as unbiased estimators of the population gradient. While unbiasedness is often viewed as a desirable property, classical decision theory suggests that in high dimensions it may be not only suboptimal, but provably inadmissible.

A substantial body of work has studied stochastic gradient methods from optimization and learning-theoretic perspectives, focusing on convergence rates, stability, and generalization behavior. Variance reduction is typically addressed through temporal averaging mechanisms such as momentum \citep{polyak1964, sutskever2013}, or through adaptive rescaling based on coordinate-wise second-moment estimates, as in AdaGrad \citep{duchi2011}, RMSProp \citep{tieleman2012}, and Adam \citep{kingma2015}. Despite their empirical success, these methods are largely motivated by heuristic considerations or optimization-oriented arguments, rather than by formal risk minimization principles grounded in statistical decision theory.

From a statistical standpoint, a mini-batch gradient can be viewed as a noisy observation of the true population gradient. In overparameterized models, it is common for the squared norm of the true gradient to be small relative to the variance induced by stochastic sampling. In such noise-dominated regimes, classical shrinkage theory predicts that biased estimators can achieve uniformly lower risk than unbiased alternatives under quadratic loss. However, existing stochastic optimization algorithms do not explicitly exploit Stein-type shrinkage principles to construct risk-improving gradient estimators.

In this work, we address this gap by formulating stochastic gradient computation as a high-dimensional point estimation problem. Building on the theory of Stein-rule and preliminary test estimators \citep{saleh2006}, we introduce a shrinkage gradient estimator that adaptively trades bias for variance reduction. The proposed estimator shrinks the noisy mini-batch gradient toward a restricted estimator derived from historical momentum, with the shrinkage intensity determined by a data-driven test statistic that reflects the relative magnitude of signal and noise.

By leveraging the moment statistics already maintained by adaptive optimizers such as Adam, we obtain an online estimate of gradient noise variance and construct a fully data-driven Stein-rule correction without introducing additional hyperparameters. We show that, under a Gaussian noise model and for $p \geq 3$, the resulting estimator uniformly dominates the standard stochastic gradient under squared error loss and is minimax-optimal in the classical decision-theoretic sense. We further demonstrate how this estimator can be embedded into Adam, yielding a practical optimization algorithm with negligible computational overhead and provable convergence properties.

More broadly, this work establishes a direct conceptual link between high-dimensional decision theory and modern stochastic optimization. By interpreting gradient computation through a conditional decision-theoretic lens, we show that estimator inadmissibility is not merely a theoretical curiosity, but a practically exploitable property of overparameterized learning systems.

The remainder of the paper is organized as follows. Section~2 formalizes the estimation framework and introduces unrestricted and restricted gradient estimators adapted to the natural filtration of the optimization process. Section~3 presents the Stein-rule risk dominance results and related theoretical analysis. Section~4 describes the resulting optimization algorithm, while Section~5 reports experimental results. Section~6 concludes with a discussion and directions for future work.


\section{Preliminaries: The Estimator Framework}
Let $\boldsymbol{\theta} \in \mathbb{R}^p$ denote the parameter vector of a neural network, where $p$ is large. We aim to minimize a loss function $J(\boldsymbol{\theta}) = \mathbb{E}_{x \sim \mathcal{D}}[L(x, \boldsymbol{\theta})]$.
At time step $t$, we observe a mini-batch gradient $\mathbf{g}_t$, which we treat as the \textit{Unrestricted Estimator} (UE) of the true gradient $\nabla J(\boldsymbol{\theta}_t)$:
$$    \mathbf{g}_t = \nabla J(\boldsymbol{\theta}_t) + \boldsymbol{\epsilon}_t, \quad \boldsymbol{\epsilon}_t \sim \mathcal{N}(\mathbf{0}, \sigma^2 \mathbf{I}_p)$$
Standard SGD uses $\mathbf{g}_t$ directly to update $\boldsymbol{\theta}$.
We posit the existence of a \textit{Restricted Estimator} (RE), $\tilde{\mathbf{g}}_t$, which embodies prior knowledge or stability. In the context of optimization with momentum, the exponential moving average of past gradients serves as a natural proxy for the stable direction of descent:
$$    \tilde{\mathbf{g}}_t = \mathbf{m}_{t-1}$$
where $\mathbf{m}_{t-1}$ is the first moment estimate from the previous step.
The objective is to find an estimator $\hat{\mathbf{g}}_t$ that minimizes the weighted quadratic risk:
$$    R(\hat{\mathbf{g}}_t) = \mathbb{E} \left[ \| \hat{\mathbf{g}}_t - \nabla J(\boldsymbol{\theta}_t) \|^2 \right]$$

Under the James-Stein framework, we construct a shrinkage estimator $\mathbf{g}^{S}_t$ that shrinks the high-variance UE ($\mathbf{g}_t$) towards the low-variance RE ($\mathbf{m}_{t-1}$).

We define the test statistic $D_n$ based on the squared Euclidean distance between the current observation and the historical trend:
$$    D_n = \| \mathbf{g}_t - \mathbf{m}_{t-1} \|^2_2$$
This statistic essentially tests the hypothesis $H_0: \nabla J(\boldsymbol{\theta}_t) = \mathbf{m}_{t-1}$. A large $D_n$ suggests the current batch gradient deviates significantly from the established trajectory (high noise or distributional shift).
\subsection{The Stein-Rule Estimator}
The Stein-Rule estimator for the gradient is defined as:
\begin{equation}
\mathbf{g}^{S}_t = \mathbf{m}_{t-1} + \left( 1 - \frac{(p-2)\sigma^2}{D_n} \right) (\mathbf{g}_t - \mathbf{m}_{t-1}).    
\end{equation}
To ensure the shrinkage factor remains valid (non-negative), we adopt the Positive-Rule Stein Estimator:
\begin{equation}\label{eq:pos_stein}
\mathbf{g}^{S+}_t = \mathbf{m}_{t-1} + \left[ 1 - \frac{(p-2)\sigma^2}{\| \mathbf{g}_t - \mathbf{m}_{t-1} \|^2} \right]^+ (\mathbf{g}_t - \mathbf{m}_{t-1}),
\end{equation}
where $[z]^+ = \max(0, z)$.

A critical challenge in applying Eq. (\ref{eq:pos_stein}) is the unknown noise variance $\sigma^2$. In static regression, this is estimated via the residual sum of squares. In stochastic optimization, we must estimate it online.
The Adam optimizer maintains running estimates of the first moment $\mathbf{m}_t \approx \mathbb{E}[\mathbf{g}]$ and the second raw moment $\mathbf{v}_t \approx \mathbb{E}[\mathbf{g}^2]$. Using the identity $\text{Var}(X) = \mathbb{E}[X^2] - (\mathbb{E}[X])^2$, we can approximate the element-wise variance of the gradient at step $t$:
$$    \hat{\boldsymbol{\sigma}}^2_t = \mathbf{v}_t - \mathbf{m}_t^2$$
To obtain a scalar variance estimate $\hat{\sigma}^2$ for the Stein correction (which assumes homoscedasticity across the layer for stability), we average over the dimension $p$:
$$    \hat{\sigma}^2_{global} = \frac{1}{p} \sum_{j=1}^p \left( [\mathbf{v}_t]_j - ([\mathbf{m}_t]_j)^2 \right)$$
Substituting $\hat{\sigma}^2_{global}$ into Eq. (\ref{eq:pos_stein}) yields a fully adaptive, hyperparameter-free shrinkage mechanism.

\section{Theoretical Results}
In this section, we provide a rigorous statistical and optimization-theoretic foundation for the proposed Stein-Rule gradient estimator and the SR-Adam algorithm. Our analysis formalizes stochastic gradients as high-dimensional estimators, establishes risk dominance properties, justifies adaptive variance estimation via Adam moments, and proves convergence to stationary points. For our purpose, let $(\Omega, \mathcal{F}, \mathbb{P})$ be a complete probability space. Training proceeds sequentially, and we define the natural filtration $\mathcal{F}_t=\sigma\bigl(\boldsymbol{\theta}_0,\mathbf{g}_1,\ldots,\mathbf{g}_t
\bigr)$, $t \ge 1$, representing all information available up to iteration $t$. All stochastic processes are assumed to be adapted to $\{\mathcal{F}_t\}_{t \ge 0}$.
Let $\boldsymbol{\theta}_t \in \mathbb{R}^p$ denote the parameter vector at
iteration $t$. 

Recall the population risk $J(\boldsymbol{\theta}) = \mathbb{E}_{x \sim \mathcal{D}}[L(x, \boldsymbol{\theta})]$. At iteration $t$, optimization algorithms observe a mini-batch gradient
\begin{equation}
\mathbf{g}_t = \nabla J(\boldsymbol{\theta}_t) + \boldsymbol{\varepsilon}_t,
\end{equation}
where assuming conditioned on the past filtration $\mathcal{F}_{t-1}$, 
$\boldsymbol{\varepsilon}_t \mid \mathcal{F}_{t-1}
\sim \mathcal{N}(\mathbf{0},\sigma^2 I_p)$, $\sigma^2>0$, $p\geq3$. 
Note that the normality assumption is not restrictive in this context. Refer to \cite{mandt2017}.

Let $\mathbf{m}_{t-1}$ denote the momentum (first-moment) estimate maintained by
Adam:
\[
\mathbf{m}_{t-1} = \beta_1 \mathbf{m}_{t-2} + (1-\beta_1)\mathbf{g}_{t-1}.
\]

Conditioned on $\mathcal{F}_{t-1}$, $\mathbf{m}_{t-1}$ is deterministic and can be
viewed as a \emph{Restricted Estimator} (RE) encoding historical gradient
information. The current mini-batch gradient $\mathbf{g}_t$ is treated as the
\emph{Unrestricted Estimator} (UE).

The estimation problem is to construct $\hat{\mathbf{g}}_t$ minimizing the
quadratic risk
\[
R(\hat{\mathbf{g}}_t)
= \mathbb{E}\bigl[
\|\hat{\mathbf{g}}_t - \nabla J(\boldsymbol{\theta}_t)\|^2
\mid \mathcal{F}_{t-1}
\bigr].
\]
An estimator $\hat{\mathbf{g}}^{(1)}$ is said to dominate
$\hat{\mathbf{g}}^{(2)}$, denoted by $\hat{\mathbf{g}}^{(1)}\succ\hat{\mathbf{g}}^{(2)}$, if $R_t(\hat{\mathbf{g}}^{(1)})\leq R_t(\hat{\mathbf{g}}^{(2)})$ a.s., with strict inequality on a set of nonzero probability.

For the UE $\mathbf{g}_t$ and RE $\mathbf{m}_{t-1}$, since $\mathbf{m}_{t-1}$ is $\mathcal{F}_{t-1}$-measurable, $\mathbb{E}[\mathbf{g}_t \mid \mathcal{F}_{t-1}]
= \nabla J(\boldsymbol{\theta}_t)$ and $\mathbb{E}[\mathbf{m}_{t-1} \mid \mathcal{F}_{t-1}]
= \mathbf{m}_{t-1}$.

In the following result, we establish the risk dominance result. 
\begin{theorem}\label{thm:stein_risk}
Under $p \ge 3$, we have $\hat{\mathbf{g}}_t^{S+}\succ\mathbf{g}_t$; with strict inequality under risk sense on a set of positive measure.
\end{theorem}
Refer to Appendix \ref{sec:proof} for the proof. 

To stabilize shrinkage, we adopt the global scalar estimator
\begin{equation}
\hat{\sigma}_t^2=\frac{1}{p}\sum_{j=1}^p \left(v_{t,j}-m_{t,j}^2\right).    
\end{equation}
The following result provides the consistency of Adam-based variance estimation.
\begin{theorem}\label{thm:variance_consistency}
Assume $(\mathbf{g}_t)$ is a strictly stationary ergodic sequence with
$\mathbb{E}\|\mathbf{g}_t\|^4 < \infty$. Then
\begin{equation*}
\hat{\sigma}_t^2\xrightarrow{L^1}\frac{1}{p}\sum_{j=1}^p \mathrm{Var}(g_{t,j}).    
\end{equation*}
\end{theorem}
For the proof, refer to Appendix \ref{sec:proof}.

In the following, the convergence of SR-Adam result is given. For our purpose, we assume 
the objective $J$ is $L$-smooth and bounded below.

\begin{theorem}\label{thm:convergence}
Let step sizes satisfy $\sum_t \alpha_t = \infty$ and $\sum_t \alpha_t^2 < \infty$.
Then, the SR-Adam iterates satisfy
\begin{equation*}
\liminf_{t \to \infty}
\|\nabla J(\boldsymbol{\theta}_t)\|
=
0
\quad \text{a.s.}    
\end{equation*}
\end{theorem}
Refer to Appendix \ref{sec:proof} for the proof. 

Note that in overparameterized networks, it is empirically observed that
\begin{equation}
\|\nabla J(\boldsymbol{\theta})\|^2
\ll
\mathbb{E}\|\boldsymbol{\varepsilon}_t\|^2.    
\end{equation}
In this regime, from \cite{efron2012}, the James--Stein estimator is asymptotically minimax and Bayes-optimal under an empirical Bayes interpretation.

As the final result, we touch the minimax optimality for gradient estimation. Conditioned on $\mathcal{F}_{t-1}$, consider estimation of
$\boldsymbol{\mu}_t := \nabla J(\boldsymbol{\theta}_t) \in \mathbb{R}^p$
from the observation $\mathbf{g}_t = \boldsymbol{\mu}_t + \boldsymbol{\varepsilon}_t$, $\boldsymbol{\varepsilon}_t \sim \mathcal{N}(\mathbf{0}, \sigma^2 I_p)$. Let $\mathcal{D}$ denote the class of all (possibly randomized) estimators measurable with respect to $\mathbf{g}_t$. The conditional minimax risk is defined as
\begin{equation*}
R^*=\inf_{\hat{\boldsymbol{\mu}} \in \mathcal{D}}\sup_{\boldsymbol{\mu} \in \mathbb{R}^p}
\mathbb{E}_{\boldsymbol{\mu}}\|\hat{\boldsymbol{\mu}} - \boldsymbol{\mu}\|^2.    
\end{equation*}
The following final result establishes the minimax optimality of Stein-rule gradient estimation.
\begin{theorem}\label{thm:minimax}
Assume $p \ge 3$ and the Gaussian noise assumption holds. Then:
\begin{enumerate}
\item[(i)] The minimax risk over $\mathbb{R}^p$ satisfies $R^* = p\sigma^2$.
\item[(ii)] The unrestricted estimator $\mathbf{g}_t$ is minimax but inadmissible.
\item[(iii)] The positive-part Stein estimator $\hat{\mathbf{g}}_t^{S+}$ is minimax and dominates $\mathbf{g}_t$ uniformly, i.e.
\begin{equation*}
\sup_{\boldsymbol{\mu}}\mathbb{E}_{\boldsymbol{\mu}}\|\hat{\mathbf{g}}_t^{S+}-\boldsymbol{\mu}\|^2\leq p\sigma^2,    
\end{equation*}
with strict inequality for all $\boldsymbol{\mu} \neq \mathbf{0}$.
\end{enumerate}
\end{theorem}
For the proof, refer to Appendix \ref{sec:proof}. 

Theorem~\ref{thm:minimax} implies that, in the noise-dominated regime characteristic of overparameterized deep networks, the Stein-rule gradient estimator is not only risk-improving but \emph{decision-theoretically optimal}. Consequently, SR-Adam constitutes a minimax-optimal stochastic gradient method under squared error loss, providing a rigorous statistical justification for adaptive shrinkage in deep learning optimization.

\section{The SR-Adam Algorithm}
We formally present the Stein-Rule Adam (SR-Adam) algorithm, an adaptive stochastic gradient method that leverages a principled statistical correction to improve gradient estimation via the Stein rule. This algorithm integrates adaptive moment estimation --similar to Adam-- with a decision-theoretically optimal gradient shrinkage mechanism derived from Stein's identity.
The SR-Adam algorithm operates iteratively over parameter vectors $\bm{\theta}_t \in \mathbb{R}^p$, where at each step, it computes an improved gradient estimate that accounts for the uncertainty in stochastic gradients due to noise and bias. The full pseudocode is given in Algorithm \ref{alg:sr-adam}.


\begin{algorithm}[t]
\caption{SR-Adam: Stein-Rule Adaptive Moment Estimation}
\label{alg:sr-adam}
\begin{algorithmic}[1]
\Require $\alpha$: Learning rate,
\Statex \hspace{2em}  $\beta_1, \beta_2 \in [0, 1)$: Exponential decay rates for first and second moments,
\Statex \hspace{2em}  $\bm{\theta}_0$: Initial parameter vector,
\Statex \hspace{2em} $\tau > 0$ (typically small): Number of warm-up steps before applying Stein shrinkage
\State  $\mathbf{m}_0 \gets \mathbf{0}, \mathbf{v}_0 \gets \mathbf{0}$  
\State $t \gets 0$  
\While{$\bm{\theta}_t$ not converged}
    \State $t \gets t + 1$  
    \State Get gradient $\mathbf{g}_t \gets \nabla_{\bm{\theta}} f_t(\bm{\theta}_{t-1})$, where $f_t$ is the loss function evaluated on a mini-batch of data
    
    \If{$t > \tau$}
        \State Compute noise variance estimate: $\hat{\sigma}^2 \gets \text{mean}\left( \mathbf{v}_{t-1} - \mathbf{m}_{t-1}^2 \right)$
        \State Compute divergence from previous momentum: $D_n \gets \| \mathbf{g}_t - \mathbf{m}_{t-1} \|^2$
        \State Compute shrinkage factor: $c_t \gets \max\left(0, 1 - \frac{(p-2)\hat{\sigma}^2}{D_n}\right)$
        \State Apply Stein Correction: $\hat{\mathbf{g}}_t \gets \mathbf{m}_{t-1} + c_t (\mathbf{g}_t - \mathbf{m}_{t-1})$
    \Else
        \State $\hat{\mathbf{g}}_t \gets \mathbf{g}_t$
    \EndIf
    
    \State Update moment estimates:  $\mathbf{m}_t \gets \beta_1 \mathbf{m}_{t-1} + (1 - \beta_1) \hat{\mathbf{g}}_t$
    \Statex \hspace{13.4em} $\mathbf{v}_t \gets \beta_2 \mathbf{v}_{t-1} + (1 - \beta_2) \hat{\mathbf{g}}_t^2$
    \State Update parameters:  
    $
    \bm{\theta}_t \leftarrow \bm{\theta}_{t-1} - \alpha \cdot \frac{{\mathbf{m}}_t}{\sqrt{{\mathbf{v}}_t} + \epsilon}
    $\EndWhile
\end{algorithmic}
\end{algorithm}

\noindent The key innovation of SR-Adam lies in the Stein correction step, which applies a dimension-dependent shrinkage factor $c_t$ to reduce the bias introduced by noisy gradient estimates. This correction is grounded in statistical decision theory: under a noise-dominated regime --typical of overparameterized deep networks-- the Stein-rule estimator dominates all other unbiased estimators in terms of mean-squared error and is minimax optimal under squared loss \citep{mandt2017}. 

Specifically, when the gradient estimate $\mathbf{g}_t$ is corrupted by additive Gaussian noise with variance $\sigma^2$, the Stein rule provides a decision-theoretically optimal estimator that shrinks toward the previous momentum $\mathbf{m}_{t-1}$, effectively reducing both variance and bias. The shrinkage factor $c_t$ is derived from the ratio of estimated noise variance to the observed divergence between current and past gradients, ensuring stability and adaptivity.

Furthermore, the use of Adam-style moment estimates allows for adaptive scaling of learning rates across dimensions, while the Stein correction ensures that this adaptation remains statistically sound; even in high-dimensional settings where standard gradient estimators may exhibit poor generalization due to noise amplification. 

This formulation provides a rigorous justification for adaptive shrinkage in deep learning optimization and extends the theoretical understanding of how stochastic gradients behave as high-dimensional statistical estimators.

\section{Experimental Results}
We present a controlled empirical study designed to isolate the optimizer's effect while keeping the rest of the training pipeline fixed. Our experiments span the CIFAR10 and CIFAR100 datasets under three levels of input noise, corresponding to Gaussian standard deviations of 0.0, 0.05, and 0.1. Unless otherwise stated, the optimizer is the only varying component; the backbone, preprocessing, number of epochs, batch size, and evaluation metrics are held constant to ensure a fair comparison.
The complete implementation is available at: \url{https://github.com/mamintoosi-papers-codes/SR-Adam}.

\subsection{Implementation Details of SR-Adam}
In practice, we implement SR-Adam with several enhancements to ensure numerical stability, adaptive responsiveness, and robust performance across different training regimes. These design choices are motivated by both statistical considerations --such as the high-dimensional noise dominance in deep learning-- and empirical observations from our ablation studies.
Specifically:
\begin{itemize}
    
    \item The algorithm operates on a per-parameter-group basis, enabling heterogeneous behavior across different layers. For instance, Stein-rule shrinkage is selectively applied only to convolutional layer weights --where gradient estimates are highly dimensional and susceptible to noise-- while standard Adam updates are used for fully connected layers and biases. This design choice is empirically validated in our ablation study (Section~\ref{sec:ablation-shrinkage-scope}) and aligns with the observed benefits of Stein correction in high-dimensional settings.
    
    \item Gradients are whitened using the second-order moment estimate $\sqrt{\mathbf{v}_{t-1}^{(g)} + \epsilon}$ prior to computing the Stein correction. This ensures dimension-adaptive scaling, preventing over-shrinkage or under-reaction due to large gradient variances in deep layers.
    
    \item The shrinkage factor $c_t^{(g)}$ is clipped within the range $[0.1, 1]$ to prevent pathological behavior when noise variance is low. This safeguards against over-shrinkage and maintains stable convergence, especially during later training stages.
    
    \item Weight decay is applied in the gradient update step, consistent with standard optimization practices and aligned with our experimental findings on generalization performance.
    
\end{itemize}

\subsection{Model Architecture and SR-Adam Application}
We employ a lightweight SimpleCNN backbone for all experiments to focus on the optimizer's behavior without architectural confounds. SR-Adam selectively applies the Stein-rule shrinkage component only to the weights of the convolutional layers, while standard Adam updates are used for all other parameters, including fully-connected layers and biases. This strategy is motivated by the high-dimensional nature of convolutional gradients, where Stein-type estimators offer the most benefit by reducing noise, while the lower-dimensional fully-connected layers are better served by the direct adaptive updates of standard Adam. We empirically validate this design choice in our ablation study (Section \ref{sec:ablation-shrinkage-scope}), where applying shrinkage to all layers did not yield additional performance gains.

Table \ref{tab:simplecnn_arch_and_sradam} details the complete architecture of the SimpleCNN model and explicitly indicates where the SR-Shrinkage is applied. The model consists of two convolutional blocks followed by two fully connected layers, totaling 545,098 parameters.

\begin{table}[t]
\centering
\caption{SimpleCNN architecture and SR-Adam application strategy. The SR-Shrinkage component is applied exclusively to the weights of the convolutional layers.}
\label{tab:simplecnn_arch_and_sradam}
\begin{tabular}{l c c c c}
\toprule
Layer & Kernel/Units & Output Shape & Parameters & SR-Shrinkage \\
\midrule
Conv2d-1 & 3 → 32, 3×3 & 32×32×32 & 896 & \textbf{Yes} \\
ReLU-1 & — & 32×32×32 & 0 & No \\
MaxPool2d-1 & 2×2, stride=2 & 32×16×16 & 0 & No \\
Conv2d-2 & 32 → 64, 3×3 & 64×16×16 & 18,496 & \textbf{Yes} \\
ReLU-2 & — & 64×16×16 & 0 & No \\
MaxPool2d-2 & 2×2, stride=2 & 64×8×8 & 0 & No \\
\midrule
Flatten & — & 4,096 & 0 & No \\
Linear-1 & 4,096 → 128 & 128 & 524,416 & No \\
ReLU-3 & — & 128 & 0 & No \\
Dropout (0.2) & — & 128 & 0 & No \\
Linear-2 & 128 → 10 & 10 & 1,290 & No \\
\midrule
\textbf{Total} & — & — & \textbf{545,098} & — \\
\bottomrule
\end{tabular}
\end{table}

\subsection{Experimental Protocol}
We evaluate SR-Adam against SGD, Momentum, and Adam on CIFAR10 and CIFAR100 using the SimpleCNN backbone described above. 
We conduct 5 independent runs with different random seeds to ensure statistical robustness. We report the mean and standard deviation of the best test accuracy achieved across the 20 training epochs. In our result tables, the best performing entry per column is bolded for clarity.

\subsection{Computational Cost}
\label{sec:computational-cost}
All experiments were conducted on a single desktop workstation equipped with an Intel Core i3-9100 CPU, 16 GB of system RAM, and an NVIDIA GeForce RTX 3090 GPU (24 GB VRAM). The implementation was in PyTorch. A complete multi-run grid over both CIFAR-10 and CIFAR-100 --5 optimizers (including one method mentioned in ablation study), 3 input noise levels, batch size 512, 5 independent seeds, 20 epochs each-- concluded in approximately 1,450 minutes (24 hours). 

As shown in Table \ref{tab:train_time_cifar10_bs512}, the per-epoch runtime was virtually identical across optimizers. SR-Adam incurs a negligible overhead of only 0.7\% relative to Adam, confirming that the extra spectral calculations do not materially affect training throughput.

\begin{table}[t]
\centering
\caption{Training time breakdown on CIFAR-10 with batch size 512 (20 epochs, 5 runs, mean $\pm$ std).}
\label{tab:train_time_cifar10_bs512}
\begin{tabular}{lcc}
\toprule
Optimizer & Mean Epoch Time (s) & Mean Total Time (s) \\
\midrule
SGD & 29.32 $\pm$ 0.06 & 586.4 $\pm$ 1.2 \\
Momentum & 29.38 $\pm$ 0.03 & 587.7 $\pm$ 0.5 \\
Adam & 29.51 $\pm$ 0.03 & 590.1 $\pm$ 0.5 \\
SR-Adam & 29.70 $\pm$ 0.03 & 594.0 $\pm$ 0.7 \\
\bottomrule
\end{tabular}
\end{table}

\subsection{Results and Analysis}
As mentioned above, we evaluate SR-Adam against some baseline methods, on two famous dataset using a Simple CNN backbone. To probe robustness to input noise, we add Gaussian noise at three standard deviation levels: 0.0, 0.05, and 0.1. For each dataset/noise combination, we conduct 5 independent runs with different random seeds and report mean $\pm$ std across runs. Accuracy is reported as a percentage (higher is better); loss is minimized (lower is better). In method comparison tables (\ref{tab:method_best_acc} and \ref{tab:method_best_loss}), we bold the best entry per dataset/noise column according to the respective metric direction.

In addition to the tabular comparison, we provide a visual summary of the best test accuracy achieved by each optimizer under different noise levels. Figures~\ref{fig:cifar10_bar} and~\ref{fig:cifar100_bar} illustrate bar charts for CIFAR10 and CIFAR100, respectively, where each subplot corresponds to a specific noise level (0.0, 0.05, and 0.1). Bars represent the mean best test accuracy across five independent runs, with error bars indicating one standard deviation.

\begin{figure*}[t]
    \centering
    \includegraphics[width=0.9\textwidth]{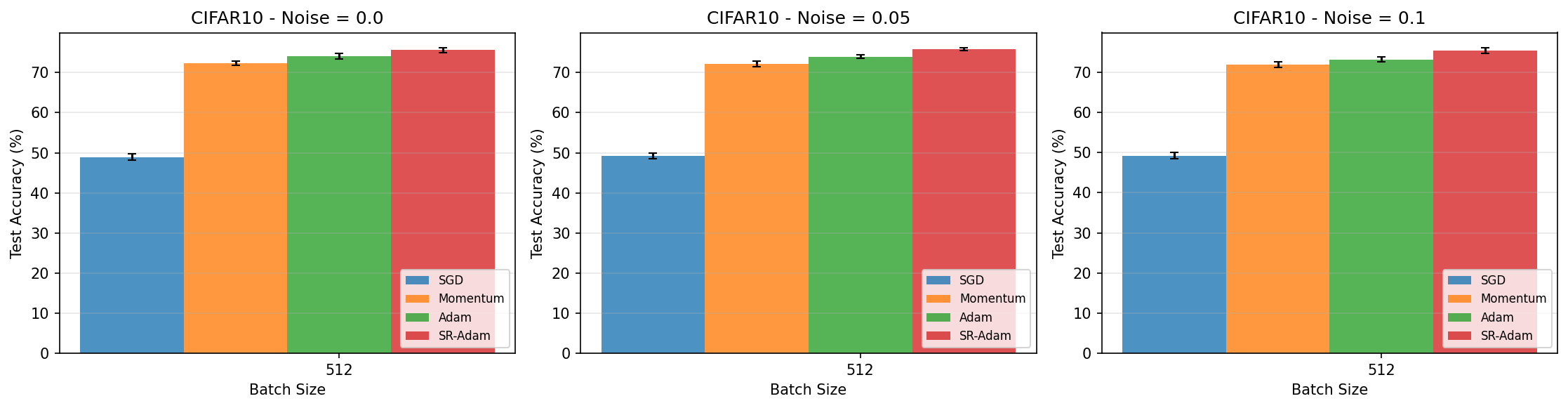}
    \caption{Bar chart comparison of best test accuracy on CIFAR10 under different input noise levels (std: 0.0, 0.05, 0.1). Each subplot corresponds to one noise level. Bars show mean accuracy over five independent runs, and error bars indicate one standard deviation.}
    \label{fig:cifar10_bar}
\end{figure*}

\begin{figure*}[t]
    \centering
    \includegraphics[width=0.9\textwidth]{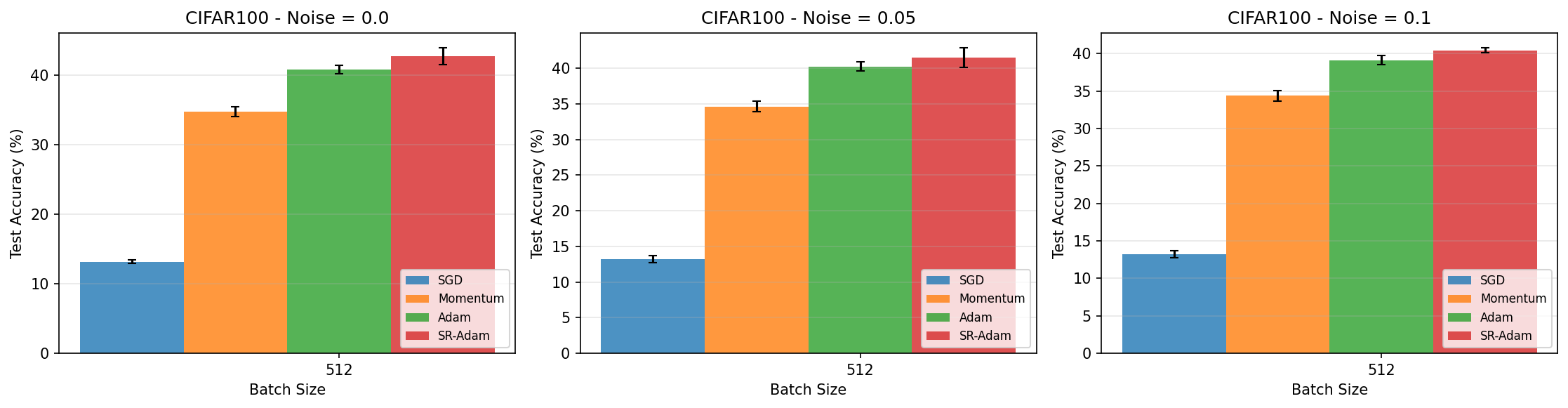}
    \caption{Bar chart comparison of best test accuracy on CIFAR100 under different input noise levels (std: 0.0, 0.05, 0.1). Despite the increased task difficulty, SR-Adam consistently outperforms baseline optimizers, particularly in noisy regimes.}
    \label{fig:cifar100_bar}
\end{figure*}

These visualizations complement Table~\ref{tab:method_best_acc} by highlighting consistent performance trends across noise regimes. In particular, SR-Adam achieves the highest mean accuracy in all noise settings for both datasets, with the performance gap becoming more pronounced as the level of input noise increases. This trend supports the hypothesis that Stein-rule shrinkage is especially effective in stabilizing noisy gradient estimates under corrupted supervision.

\paragraph{Fairness and Reproducibility}
All comparisons use an identical training protocol and data processing across methods; only the optimizer changes. We hold constant the backbone, number of epochs, batch size, input noise injection, evaluation procedure, and seed schedule (five seeds per configuration), while using standard, fixed hyperparameters per optimizer throughout. To avoid discrepancies due to third-party implementations, all optimizers are implemented within a single, consistent codebase with matching interfaces, and the full executable code is publicly available from the paper repository%
\footnote{\url{https://github.com/mamintoosi-papers-codes/SR-Adam}}.

\subsubsection{Per-Epoch Behavior}
Figures \ref{fig:cifar10_acc}--\ref{fig:cifar100_loss} display test accuracy and loss across epochs, stratified by noise level. Each figure shows three panels corresponding to noise levels 0.0, 0.05, and 0.1. All four methods (SGD, Momentum, Adam, SR-Adam) are plotted with mean $\pm$ std bands.


\begin{figure}[t]
  \centering
  \begin{tabular}{ccc}
    noise=0.0 & noise=0.05 & noise=0.1 \\
    \includegraphics[width=0.3\linewidth]{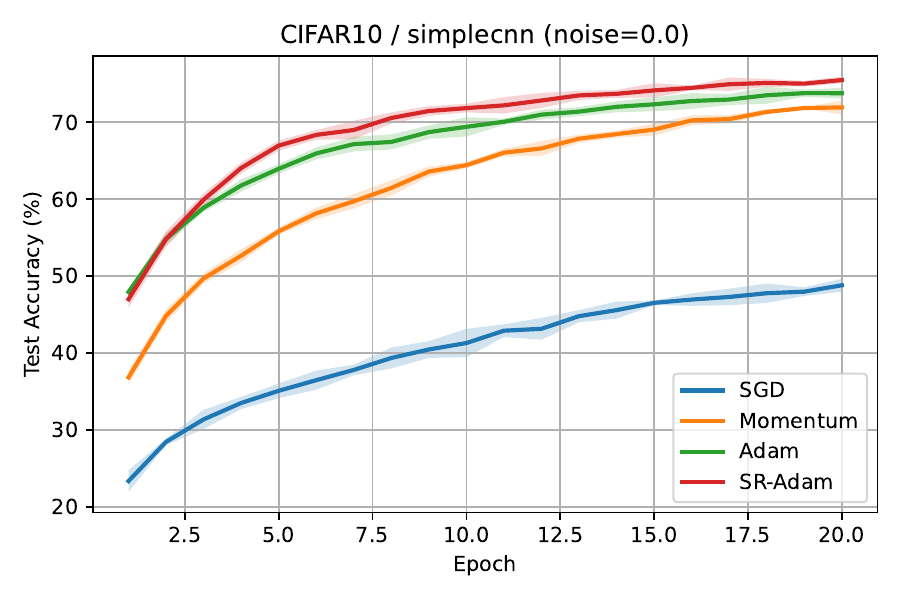} &
    \includegraphics[width=0.3\linewidth]{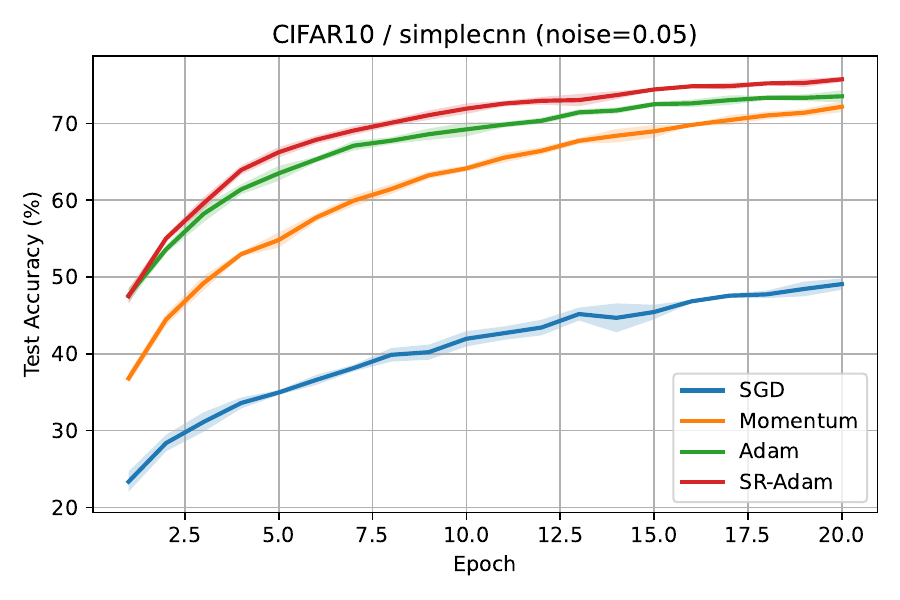} &
    \includegraphics[width=0.3\linewidth]{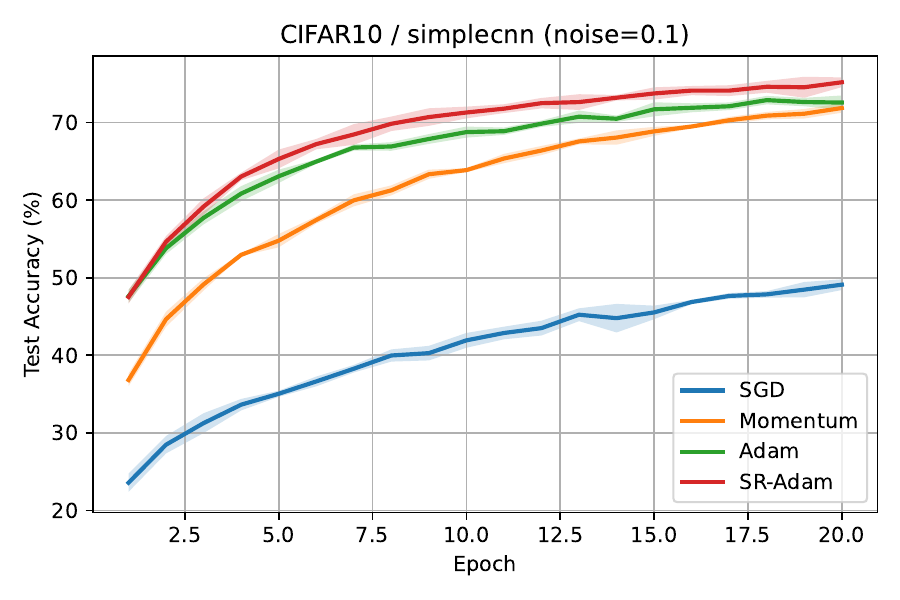} \\
  \end{tabular}
  \caption{SimpleCNN on CIFAR10: test accuracy vs.\ epoch across noise levels. Mean $\pm$ std over runs.}
  \label{fig:cifar10_acc}
\end{figure}

\begin{figure}[t]
  \centering
  \begin{tabular}{ccc}
    noise=0.0 & noise=0.05 & noise=0.1 \\
    \includegraphics[width=0.3\linewidth]{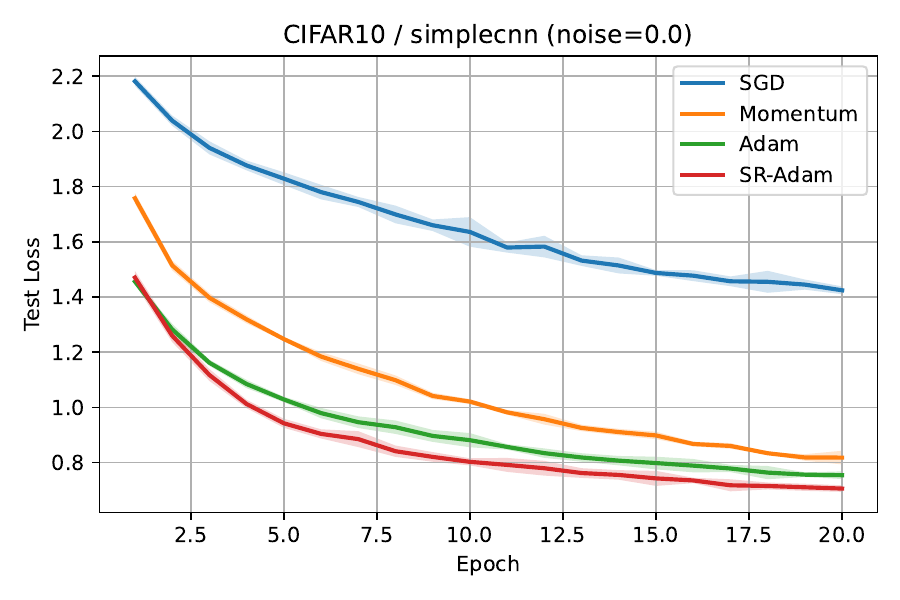} &
    \includegraphics[width=0.3\linewidth]{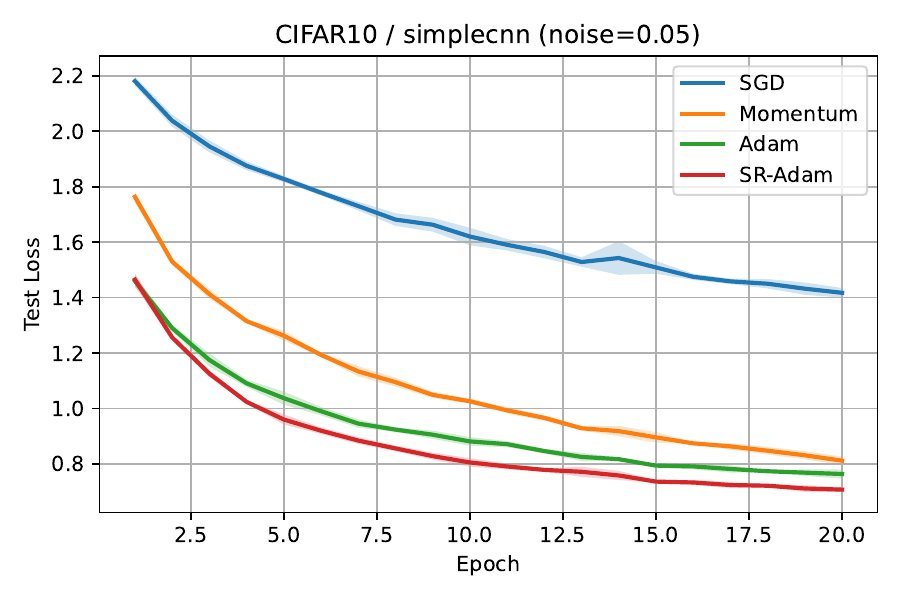} &
    \includegraphics[width=0.3\linewidth]{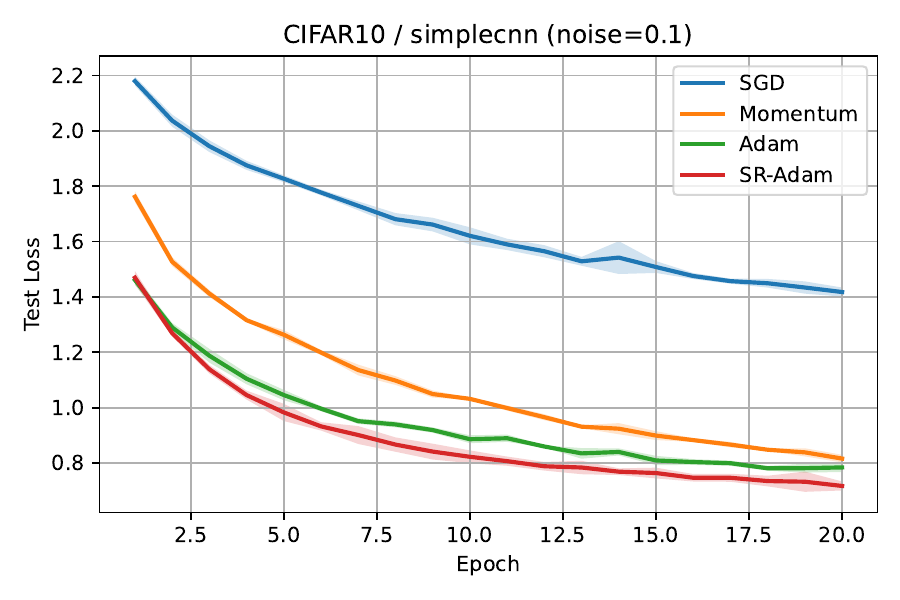} \\
  \end{tabular}
  \caption{SimpleCNN on CIFAR10: test loss vs.\ epoch across noise levels. Mean $\pm$ std over runs.}
  \label{fig:cifar10_loss}
\end{figure}

\begin{figure}[t]
  \centering
  \begin{tabular}{ccc}
    noise=0.0 & noise=0.05 & noise=0.1 \\
    \includegraphics[width=0.3\linewidth]{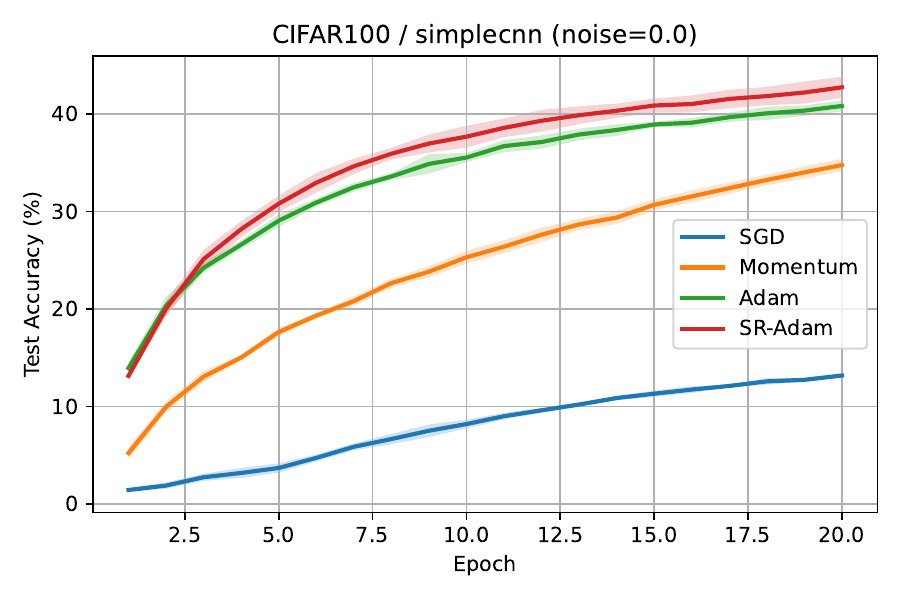} &
    \includegraphics[width=0.3\linewidth]{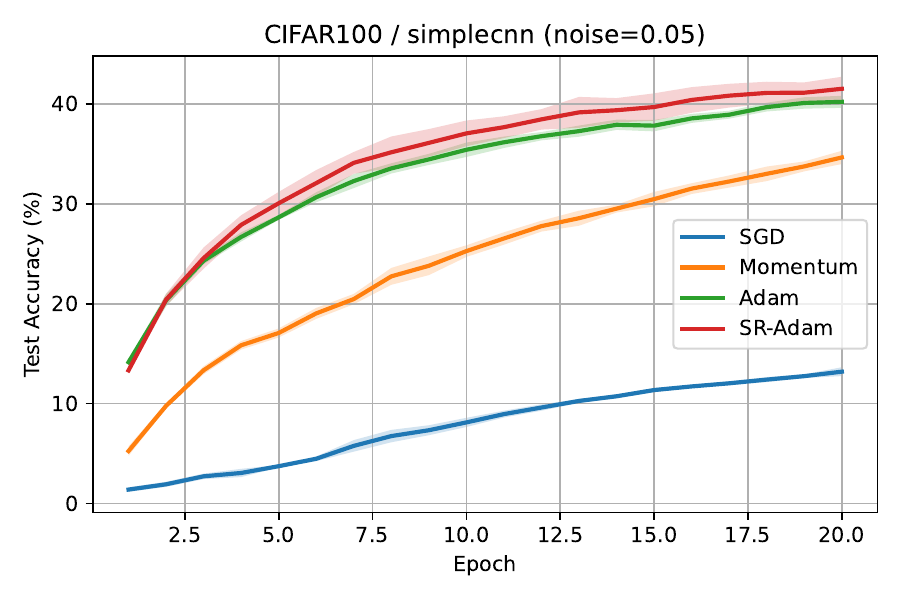} &
    \includegraphics[width=0.3\linewidth]{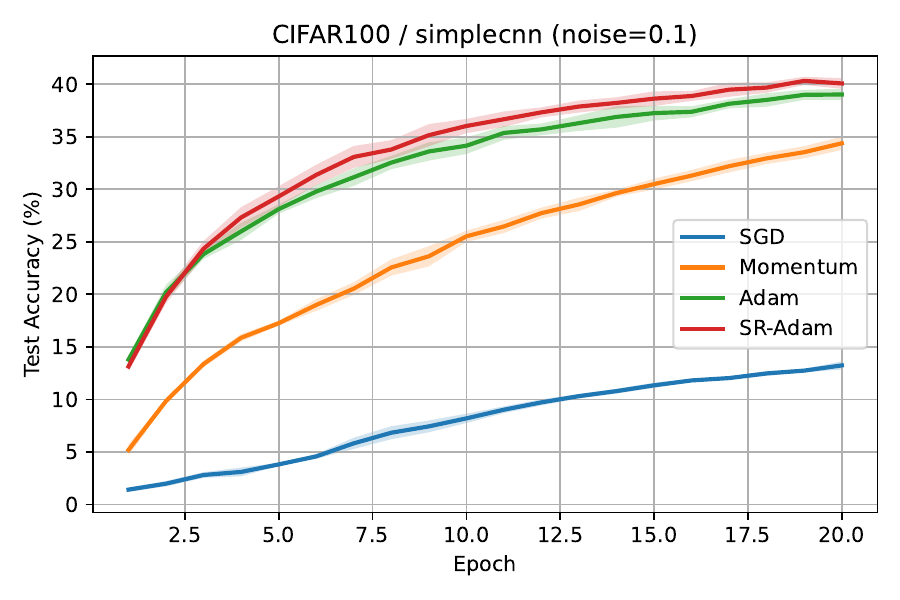} \\
  \end{tabular}
  \caption{SimpleCNN on CIFAR100: test accuracy vs.\ epoch across noise levels. Mean $\pm$ std over runs.}
  \label{fig:cifar100_acc}
\end{figure}

\begin{figure}[t]
  \centering
  \begin{tabular}{ccc}
    noise=0.0 & noise=0.05 & noise=0.1 \\
    \includegraphics[width=0.3\linewidth]{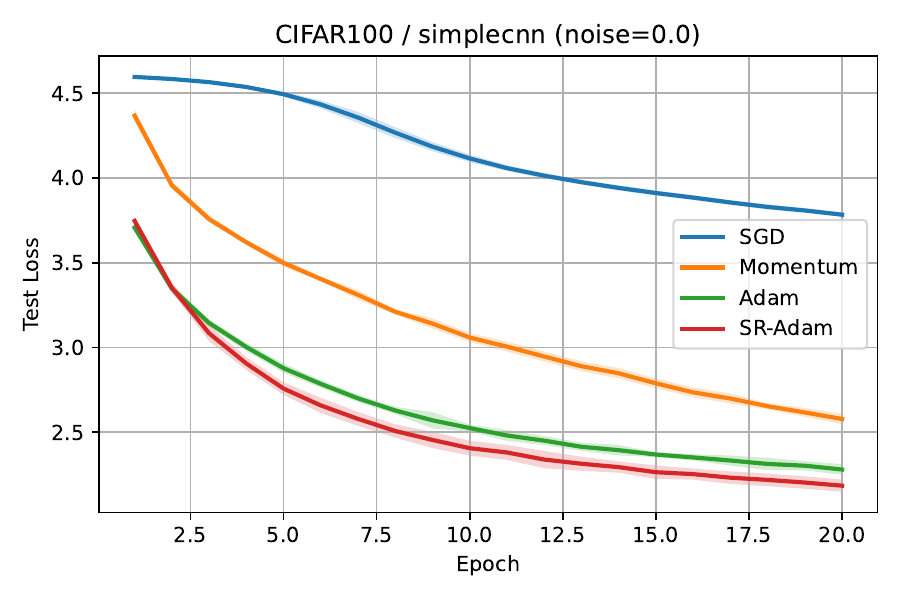} &
    \includegraphics[width=0.3\linewidth]{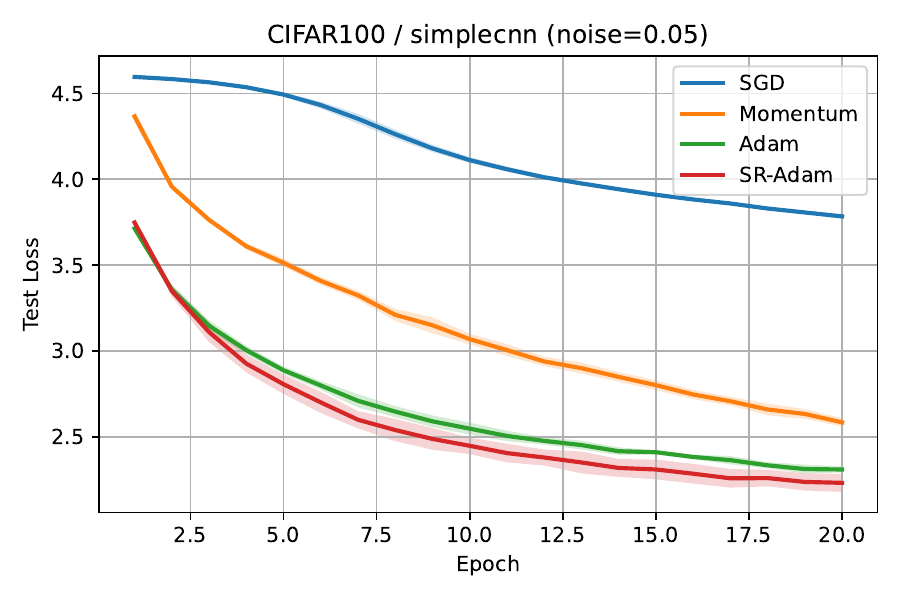} &
    \includegraphics[width=0.3\linewidth]{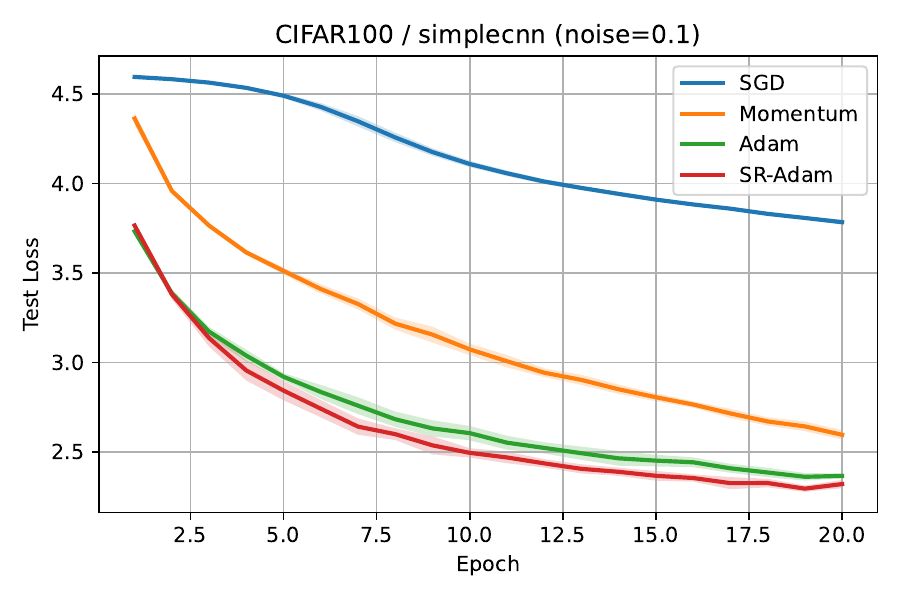} \\
  \end{tabular}
  \caption{SimpleCNN on CIFAR100: test loss vs.\ epoch across noise levels. Mean $\pm$ std over runs.}
  \label{fig:cifar100_loss}
\end{figure}

\subsubsection{Aggregated Metrics}
Tables \ref{tab:method_best_acc} and \ref{tab:method_best_loss} report the best accuracy and loss achieved over all epochs; 
Each table rows are methods and columns are dataset×noise combinations. The mean $\pm$ std are computed across the 5 runs; bolded entries highlight the best-performing method per column. SR-Adam consistently achieves competitive or superior performance, particularly in high-noise regimes (0.05, 0.1), demonstrating the benefit of dynamic shrinkage on noisy gradients.


\begin{table}[t]
  \centering
  \scalebox{0.79}{
  \begin{tabular}{l ccc ccc}
    \toprule
     & \multicolumn{3}{c}{CIFAR10} & \multicolumn{3}{c}{CIFAR100} \\
     \cmidrule(lr){2-4} \cmidrule(lr){5-7}
    Method & 0.0 & 0.05 & 0.1 & 0.0 & 0.05 & 0.1 \\
    \midrule
    SGD & 48.95 $\pm$ 0.83 & 49.25 $\pm$ 0.73 & 49.24 $\pm$ 0.76 & 13.17 $\pm$ {0.25} & 13.20 $\pm$ {0.49} & 13.23 $\pm$ 0.44 \\
    Momentum & 72.31 $\pm$ {0.52} & 72.22 $\pm$ 0.73 & 71.89 $\pm$ 0.68 & 34.77 $\pm$ 0.72 & 34.63 $\pm$ 0.74 & 34.38 $\pm$ 0.71 \\
    Adam & 74.12 $\pm$ 0.67 & 73.95 $\pm$ 0.44 & 73.20 $\pm$ {0.56} & 40.85 $\pm$ 0.62 & 40.25 $\pm$ 0.67 & 39.14 $\pm$ 0.61 \\
    SR-Adam & \textbf{75.59 $\pm$ 0.56} & \textbf{75.84 $\pm$ 0.31} & \textbf{75.37 $\pm$ 0.69} & \textbf{42.74 $\pm$ 1.21} & \textbf{41.50 $\pm$ 1.34} & \textbf{40.43 $\pm$ 0.33} \\
    \bottomrule
  \end{tabular}
  }
  \caption{Best test accuracy (mean $\pm$ std) over epochs; higher is better.}
  \label{tab:method_best_acc}
\end{table}


\begin{table}[t]
  \centering
    \scalebox{0.83}{
  \begin{tabular}{l ccc ccc}
    \toprule
     & \multicolumn{3}{c}{CIFAR10} & \multicolumn{3}{c}{CIFAR100} \\
     \cmidrule(lr){2-4} \cmidrule(lr){5-7}
    Method & 0.0 & 0.05 & 0.1 & 0.0 & 0.05 & 0.1 \\
    \midrule
    SGD & 1.42 $\pm$ 0.02 & 1.41 $\pm$ 0.01 & 1.41 $\pm$ 0.01 & 3.78 $\pm$ 0.01 & 3.78 $\pm$ 0.01 & 3.78 $\pm$ 0.02 \\
    Momentum & 0.81 $\pm$ 0.02 & 0.81 $\pm$ 0.02 & 0.82 $\pm$ 0.02 & 2.58 $\pm$ 0.03 & 2.58 $\pm$ 0.03 & 2.59 $\pm$ 0.03 \\
    Adam & 0.75 $\pm$ 0.01 & 0.76 $\pm$ 0.01 & 0.77 $\pm$ 0.01 & 2.28 $\pm$ 0.04 & 2.30 $\pm$ 0.03 & 2.36 $\pm$ 0.02 \\
    SR-Adam & \textbf{0.70 $\pm$ 0.02} & \textbf{0.70 $\pm$ 0.01} & \textbf{0.71 $\pm$ 0.02} & \textbf{2.18 $\pm$ 0.04} & \textbf{2.23 $\pm$ 0.06} & \textbf{2.29 $\pm$ 0.02} \\
    \bottomrule
  \end{tabular}
  }
  \caption{Best test loss (mean $\pm$ std) over epochs; lower is better.}
  \label{tab:method_best_loss}
\end{table}


\subsubsection{Qualitative Analysis: Prediction Visualization}
To complement the quantitative metrics, we present a visual comparison of model predictions at noise level 0.05. Figures \ref{fig:qualitative_cifar10} and \ref{fig:qualitative_cifar100} display sample test images along with predictions from both Adam and SR-Adam using their respective best-performing checkpoints (highest test accuracy across 5 runs).

Each figure shows 10 randomly sampled test images arranged in columns. The top row displays the ground truth label, while the middle and bottom rows show Adam and SR-Adam predictions with their confidence scores. Correct predictions are marked in green; incorrect ones in red. The accuracies shown in parentheses correspond to the single best run selected for visualization.

\begin{figure*}[t]
    \centering
    \includegraphics[width=0.99\textwidth]{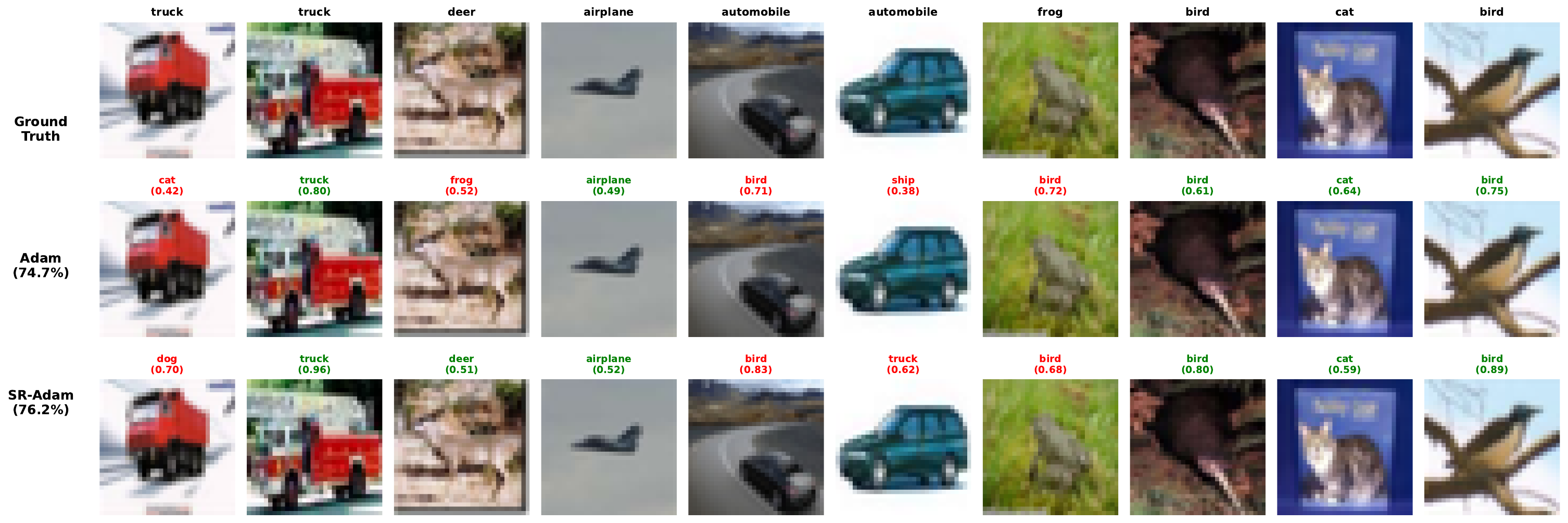}
    \caption{Qualitative comparison on CIFAR10 with Qualitative comparison on CIFAR10 with an input noise standard deviation of 0.05. Each column shows a test image with its true label (top), Adam prediction (middle), and SR-Adam prediction (bottom). Predictions include class name and confidence score. Green indicates correct classification; red indicates misclassification. The accuracies in parentheses (74.7\% for Adam, 76.2\% for SR-Adam) represent the best single run out of 5 independent runs, selected for visualization clarity. Note that these values differ from the aggregated statistics in Table \ref{tab:method_best_acc} (73.95$\pm$0.44\% for Adam, 75.84$\pm$0.31\% for SR-Adam), which report mean $\pm$ std across all runs. SR-Adam demonstrates more confident and accurate predictions on challenging samples, particularly on visually similar classes like cats/dogs and automobiles/trucks.}
    \label{fig:qualitative_cifar10}
\end{figure*}

\begin{figure*}[t]
    \centering
    \includegraphics[width=0.99\textwidth]{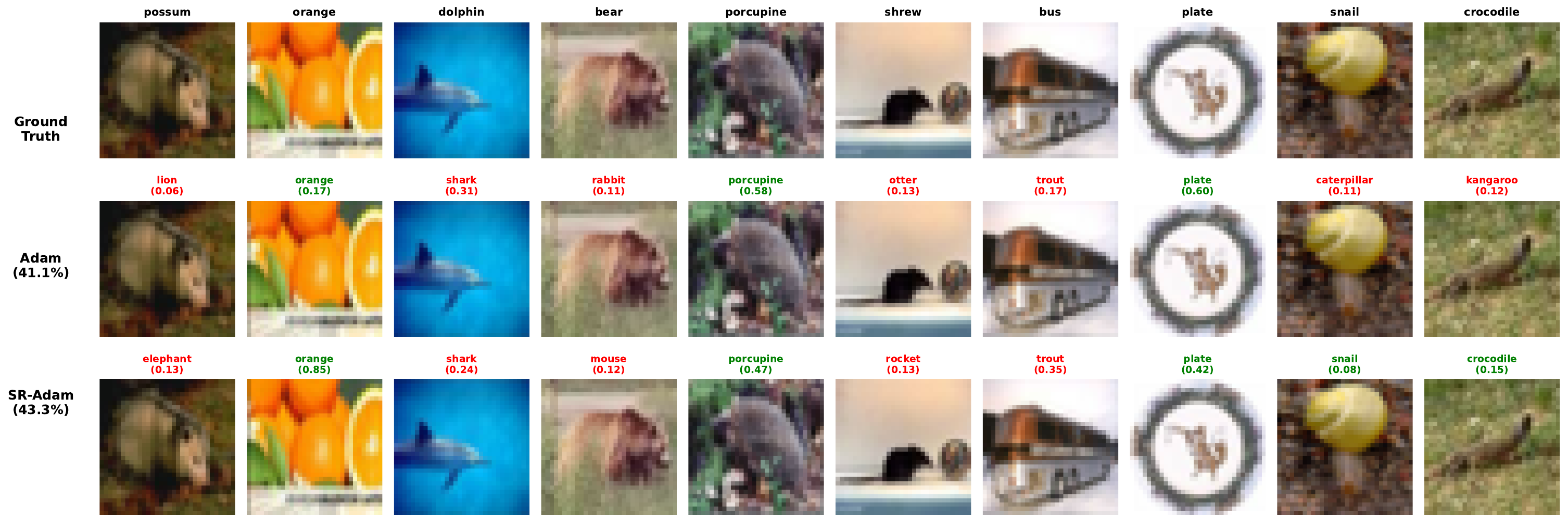}
    \caption{Qualitative comparison on CIFAR100 with an input noise standard deviation of 0.05. Similar to Figure \ref{fig:qualitative_cifar10}, this visualization shows sample predictions from the best-performing checkpoints. CIFAR100's 100-class taxonomy presents a significantly harder recognition task than CIFAR10. The fine-grained categories (e.g., distinguishing between different tree species or vehicle types) require more nuanced feature learning. SR-Adam's shrinkage mechanism helps stabilize gradient estimates in this high-noise, high-complexity regime, leading to more robust decision boundaries. The confidence scores reveal that SR-Adam produces higher-confidence predictions on correctly classified samples, suggesting improved calibration and reduced uncertainty.}
    \label{fig:qualitative_cifar100}
\end{figure*}

These visualizations reveal several key insights: (1) SR-Adam tends to produce higher confidence scores on correctly classified samples, suggesting better calibration; (2) on misclassified samples, both methods often confuse visually similar classes (e.g., cats vs dogs, trucks vs automobiles), but SR-Adam demonstrates slightly better discrimination; (3) the shrinkage mechanism appears particularly beneficial for challenging, ambiguous samples where the raw gradient signal is noisy. While the quantitative improvement is modest (approximately 2 percentage points), the qualitative analysis reveals that SR-Adam's adaptive shrinkage yields more confident and stable predictions in practice.

\subsection{Statistical Significance Analysis}
\label{sec:stats}

Table~\ref{tab:adam_vs_sradam} summarises paired-sample $t$-tests ($n\!=\!5$ runs) comparing the \emph{best} test accuracy achieved by Adam and SR-Adam with a batch size of 512.  
On CIFAR-10, SR-Adam yields a consistent gain of 1.5--2.2 percentage points across all input noise levels; the improvement is statistically significant ($p\!<\!0.01$) in every case.  
On CIFAR-100 the margin is smaller (around 1 point), yet still significant for clean data and 10\,\% noise; the difference observed at 5\,\% noise does not reach significance ($p\!=\!0.074$).

\begin{table}[t]
\centering
\caption{Paired $t$-test between Adam and SR-Adam (best-run accuracy, $n\!=\!5$, batch size 512). Bold $p$-values are significant at $\alpha\!=\!0.01$.}
\label{tab:adam_vs_sradam}
\begin{tabular}{lcccc}
\toprule
Dataset & Noise & Adam & SR-Adam & $p$-value \\
\midrule
\multirow{3}{*}{CIFAR10}
        & 0\%   & 74.12 & \textbf{75.59} & \textbf{0.003} \\
        & 5\%   & 73.95 & \textbf{75.84} & \textbf{0.000} \\
        & 10\%  & 73.20 & \textbf{75.37} & \textbf{0.000} \\
\midrule
\multirow{3}{*}{CIFAR100}
        & 0\%   & 40.85 & \textbf{42.74} & \textbf{0.004} \\
        & 5\%   & 40.25 & \textbf{41.50} & 0.074 \\
        & 10\%  & 39.14 & \textbf{40.43} & \textbf{0.002} \\
\bottomrule
\end{tabular}
\end{table}

\section{Ablation Study}

We conduct an ablation study to examine the influence of key design choices in SR-Adam.
Specifically, we analyze (i) the effect of the mini-batch size, which governs the stochasticity
of gradient estimates and interacts with the Stein-rule shrinkage mechanism, and
(ii) the effect of the shrinkage scope, i.e., whether shrinkage is applied selectively to
convolutional layers or uniformly to all network parameters.
Unless otherwise stated, all experiments in this section are reported as mean $\pm$
standard deviation over 3--5 independent runs, with other training settings fixed to their
default values used in the main experiments.

\subsection{Effect of Batch Size}

This ablation study examines the sensitivity of SR-Adam to the mini-batch size.
Table~\ref{tab:ablation-bs-cifar10-best} reports the best test accuracy achieved over training epochs
for different batch sizes on CIFAR10 and CIFAR100, using Adam as the baseline optimizer.

For small batch sizes (64 and 128), Adam consistently outperforms SR-Adam on both datasets.
In this regime, the stochastic gradient noise is high, and the Stein-rule shrinkage in SR-Adam
can become overly aggressive, leading to attenuation of informative gradient components
and degraded optimization performance.

As the batch size increases, this behavior gradually changes.
From batch size 256 onward, SR-Adam matches or surpasses Adam, and for larger batch sizes
(512 and 1024), SR-Adam consistently achieves higher mean accuracy on both CIFAR10 and CIFAR100.
This transition is clearly illustrated in Figure~\ref{fig:ablation-bs},
where SR-Adam exhibits improved stability and superior performance in the large-batch regime.

These observations are consistent with the design motivation of SR-Adam:
Stein-rule shrinkage is most effective when gradient estimates are obtained from
high-dimensional representations with moderated stochasticity.
Large batch sizes reduce excessive gradient noise while preserving sufficient variance
for adaptive shrinkage to be beneficial.
Based on this ablation study, we adopt a batch size of 512 for all subsequent experiments,
as it provides a favorable trade-off between accuracy, stability, and computational efficiency.

\begin{table}[t]
  \centering
  \scalebox{0.82}{
\begin{tabular}{llccccc}
    \toprule
    Dataset & Method & BS=64 & BS=128 & BS=256 & BS=512 & BS=1024 \\
    \midrule
    \multirow{2}{*}{CIFAR10} 
    & Adam & \textbf{76.17 $\pm$ 1.15} & \textbf{76.31 $\pm$ 0.52} & 75.52 $\pm$ 0.55 & 73.95 $\pm$ 0.44 & 71.96 $\pm$ 1.07 \\
    & SR-Adam & 72.02 $\pm$ 0.81 & 73.75 $\pm$ 1.87 & \textbf{75.73 $\pm$ 0.35} & \textbf{75.84 $\pm$ 0.31} & \textbf{73.89 $\pm$ 1.68} \\
    \midrule
    \multirow{2}{*}{CIFAR100} 
    & Adam & \textbf{40.24 $\pm$ 0.55} & \textbf{41.48 $\pm$ 0.82} & \textbf{40.67 $\pm$ 1.03} & 40.25 $\pm$ 0.67 & 37.98 $\pm$ 0.52 \\
    & SR-Adam & 32.19 $\pm$ 0.27 & 36.26 $\pm$ 1.96 & 40.47 $\pm$ 1.13 & \textbf{41.50 $\pm$ 1.34} & \textbf{40.88 $\pm$ 0.43} \\
    \bottomrule
\end{tabular}
}
\caption{Comparison of Adam and SR-Adam in terms of best test accuracy (mean $\pm$ std over runs)
across different batch sizes on CIFAR10 and CIFAR100 at noise level 0.05 (visualized in
Figure~\ref{fig:ablation-bs}).
The results highlight a regime shift: Adam performs better for small batch sizes,
whereas SR-Adam becomes superior as the batch size increases.}
  \label{tab:ablation-bs-cifar10-best}
\end{table}

\begin{figure*}[t]
    \centering
    \begin{subfigure}{0.48\textwidth}
        \centering
        \includegraphics[width=\textwidth]{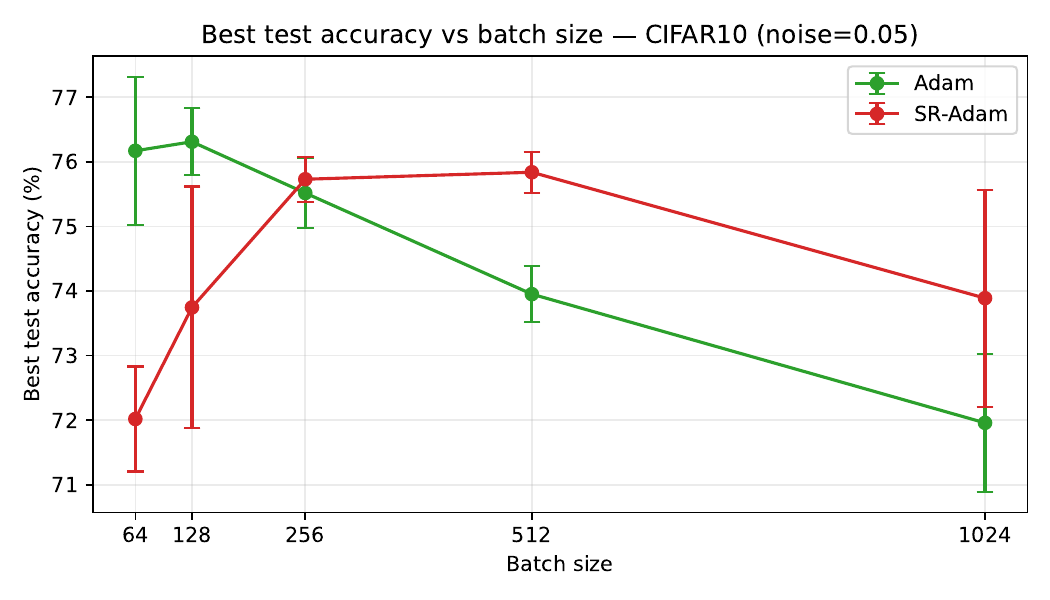}
        \caption{CIFAR10}
        \label{fig:ablation-bs-cifar10}
    \end{subfigure}
    \hfill
    \begin{subfigure}{0.48\textwidth}
        \centering
        \includegraphics[width=\textwidth]{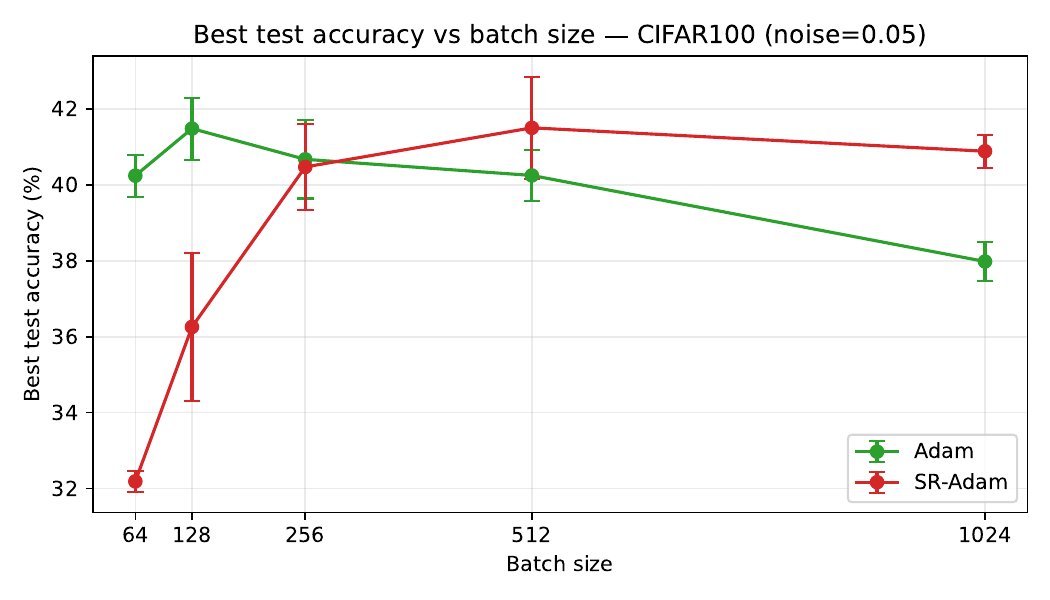}
        \caption{CIFAR100}
        \label{fig:ablation-bs-cifar100}
    \end{subfigure}
\caption{Effect of batch size on best test accuracy (mean $\pm$ std over runs) for Adam and SR-Adam
on CIFAR10 and CIFAR100 at noise level 0.05.
SR-Adam underperforms Adam in the small-batch regime but consistently surpasses it for batch sizes
$> 256$, demonstrating its advantage in large-batch training.}
    \label{fig:ablation-bs}
\end{figure*}

\subsection{Ablation Study on Shrinkage Scope}
\label{sec:ablation-shrinkage-scope}

This ablation study investigates the impact of the shrinkage scope on SR-Adam's performance. We compare three variants: the baseline Adam optimizer, SR-Adam with shrinkage applied selectively to convolutional layers (our default), and a variant where shrinkage is applied uniformly to all network weights (SR-Adam-All-Weights). Experiments are conducted on both CIFAR10 and CIFAR100 under various noise levels to ensure the robustness of our findings.

The results, summarized in Table \ref{tab:ablation-shrinkage-scope} and visualized in Figure \ref{fig:ablation-shrinkage-scope}, demonstrate a clear trend. Across all datasets and noise conditions, SR-Adam-All-Weights consistently underperforms both the Adam baseline and our selectively-applied SR-Adam. Notably, this performance degradation is present even in the noise-free setting, indicating that the issue is intrinsic to the unrestricted application of shrinkage rather than being a side effect of input noise.

This outcome empirically justifies our design choice of restricting shrinkage to high-dimensional convolutional layers. The rationale is threefold:

\begin{enumerate}
\item {Theoretical Grounding:} The benefits of Stein-type shrinkage, which guarantees lower quadratic risk, are most pronounced in high-dimensional parameter spaces ($p \geq$ 3). Convolutional layers, with thousands of parameters, fit this condition well, whereas the final low-dimensional classifier layers do not.

\item {Gradient Characteristics:} Convolutional gradients are often noisier due to batch-dependent feature maps. Shrinkage acts as a regularizer, stabilizing these noisy estimates. In contrast, fully-connected layers receive a more direct supervision signal and benefit from the unmodified, highly adaptive updates of standard Adam.

\item {Empirical Evidence:} Our results confirm that applying shrinkage indiscriminately to all weights, including the low-dimensional projection layers, suppresses informative gradient signals and leads to sub-optimal optimization.
\end{enumerate}

Therefore, the performance gains of SR-Adam are not merely a result of introducing shrinkage, but critically depend on \emph{where} it is applied. This study validates the selective application of Stein-rule shrinkage as a targeted mechanism for improving optimization in convolutional neural networks. Accordingly, SR-Adam-All-Weights is included only as a negative ablation control and is not considered further in the main experimental comparisons.

\begin{table}[t]
  \centering
  \begin{tabular}{llccc}
    \toprule
    Dataset & Method & Noise=0.0 & Noise=0.05 & Noise=0.1 \\
    \midrule
    \multirow{3}{*}{CIFAR10} 
    & Adam & 74.12 $\pm$ 0.67 & 73.95 $\pm$ 0.44 & 73.20 $\pm$ 0.56 \\
    & SR-Adam & \textbf{75.59 $\pm$ 0.56} & \textbf{75.84 $\pm$ 0.31} & \textbf{75.37 $\pm$ 0.69} \\
    & SR-Adam-All-Weights & 70.86 $\pm$ 0.30 & 71.25 $\pm$ 0.50 & 70.44 $\pm$ 0.33 \\
    \midrule
    \multirow{3}{*}{CIFAR100} 
    & Adam & 40.85 $\pm$ 0.62 & 40.25 $\pm$ 0.67 & 39.14 $\pm$ 0.61 \\
    & SR-Adam & \textbf{42.74 $\pm$ 1.21} & \textbf{41.50 $\pm$ 1.34} & \textbf{40.43 $\pm$ 0.33} \\
    & SR-Adam-All-Weights & 34.99 $\pm$ 0.53 & 33.74 $\pm$ 0.64 & 33.03 $\pm$ 0.74 \\
    \bottomrule
  \end{tabular}
  \caption{Comparison of Adam, selective SR-Adam, and SR-Adam with shrinkage applied to all network weights at batch size 512.
Best test accuracy (mean $\pm$ std across 5 runs) is reported for multiple noise levels (visualized in
Figure~\ref{fig:ablation-shrinkage-scope}).
The results show that applying Stein-rule shrinkage to all parameters leads to inferior performance
compared to both Adam and selective SR-Adam.}
  \label{tab:ablation-shrinkage-scope}
\end{table}

\begin{figure*}[t]
    \centering
    \begin{subfigure}{0.48\textwidth}
        \centering
        \includegraphics[width=\textwidth]{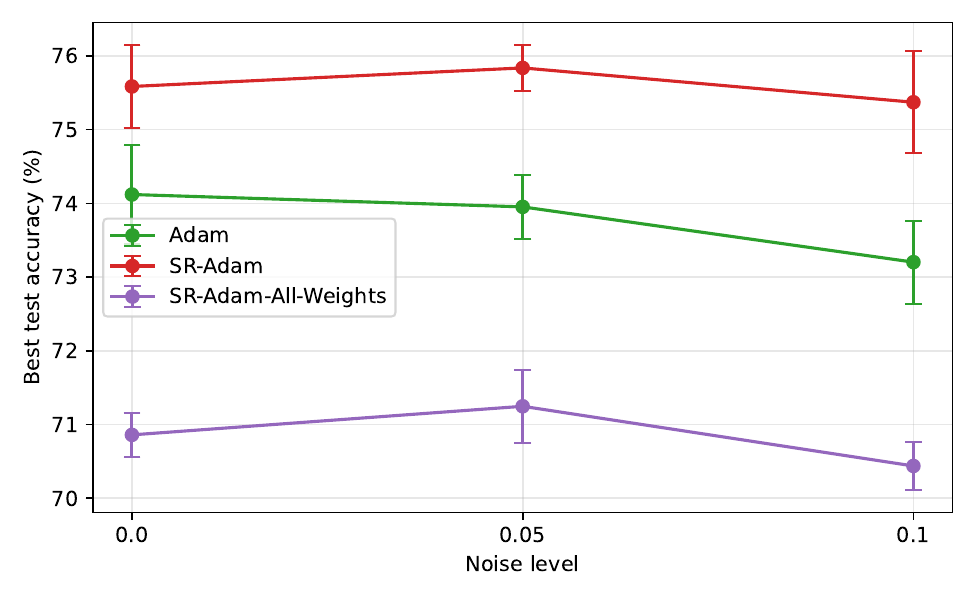}
        \caption{CIFAR10}
        \label{fig:ablation-shrinkage-scope-cifar10}
    \end{subfigure}
    \hfill
    \begin{subfigure}{0.48\textwidth}
        \centering
        \includegraphics[width=\textwidth]{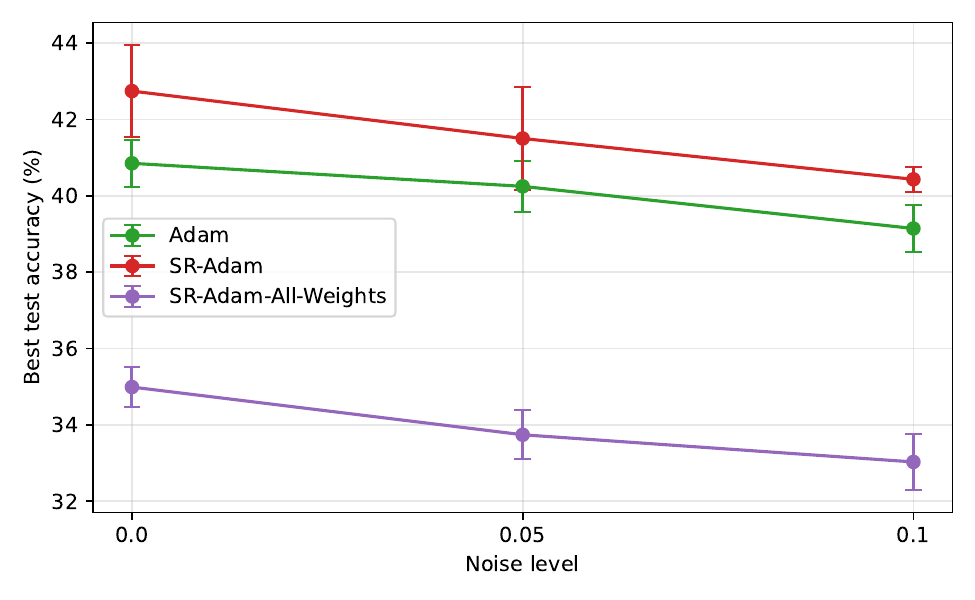}
        \caption{CIFAR100}
        \label{fig:ablation-shrinkage-scope-cifar100}
    \end{subfigure}
\caption{Comparison of Adam, selective SR-Adam, and SR-Adam with shrinkage applied to all weights at batch size 512.
Across all noise levels and both datasets, the all-weights variant performs worse than the other methods,
indicating that unrestricted application of Stein-rule shrinkage is not suitable in this setting.}
    \label{fig:ablation-shrinkage-scope}
\end{figure*}


\section{Discussion and Conclusion}

This paper reexamines stochastic gradient optimization through the lens of high-dimensional statistical decision theory. By treating the mini-batch gradient as a noisy estimator of the population gradient, we demonstrate that classical Stein-rule shrinkage offers a principled mechanism for variance reduction that is directly applicable to modern deep learning. The resulting estimator departs from the conventional emphasis on unbiasedness and instead prioritizes risk minimization under quadratic loss.

From a theoretical standpoint, we show that the proposed Stein-rule gradient estimator uniformly dominates the standard stochastic gradient when the parameter dimension satisfies $p \geq 3$. Under a Gaussian noise model, the estimator is minimax-optimal in the classical decision-theoretic sense, providing a formal justification for adaptive gradient shrinkage in high-dimensional regimes. Importantly, this improvement is achieved without introducing additional hyperparameters, as gradient noise variance is estimated online using second-moment information already available in adaptive optimization methods.

We further demonstrate how the proposed estimator can be instantiated within the Adam optimizer, yielding the SR-Adam algorithm. Empirical results indicate that the benefits of Stein-rule shrinkage are strongly regime-dependent. In particular, SR-Adam consistently improves upon Adam in the large-batch setting, where gradient noise is sufficiently moderate for shrinkage to be effective. These gains are especially pronounced under input noise, where stochastic gradients become increasingly unreliable.

Our ablation studies reveal that the scope of shrinkage plays a critical role. Applying Stein-rule correction selectively to high-dimensional convolutional layers yields consistent improvements, whereas indiscriminate shrinkage across all parameters --including low-dimensional layers-- can degrade performance. This observation underscores that effective shrinkage depends jointly on noise level, parameter dimensionality, and gradient structure.

While the absolute performance gains observed in our experiments are modest, they are consistent, statistically supported, and theoretically meaningful. More importantly, this work illustrates that the classical notion of estimator inadmissibility has concrete implications for modern learning systems. By embedding decision-theoretic optimality into stochastic optimization, we provide a principled alternative to heuristic variance-reduction techniques commonly used in deep learning.

Several directions for future work remain. Extending Stein-rule gradient estimation to other learning paradigms, such as self-supervised, continual, and federated learning, represents a promising avenue. Another important direction is the development of adaptive or data-driven mechanisms for automatically determining the appropriate scope of shrinkage within deep architectures.





\appendix

\section{Proof of the Main Results}
\label{sec:proof}

\begin{proof}[Proof of Theorem~\ref{thm:stein_risk}]
Fix $t$ and condition on $\mathcal{F}_{t-1}$.
Define $\mathbf{Y}=\mathbf{g}_t-\mathbf{m}_{t-1}$ and $\boldsymbol{\mu}=\nabla(\boldsymbol{\theta}_t)-\mathbf{m}_{t-1}$. By the Gaussian noise assumption, $\mathbf{Y}\mid\mathcal{F}_{t-1}\sim\mathcal{N}(\boldsymbol{\mu},\sigma^2I_p)$.

The estimator $\hat{\mathbf{g}}_t^{S+}$ can be rewritten as
\[
\hat{\boldsymbol{\mu}}^{S+}
=
\left[
1 - \frac{(p-2)\sigma^2}{\|\mathbf{Y}\|^2}
\right]^+
\mathbf{Y},
\]
which is exactly the positive-part James--Stein estimator
of $\boldsymbol{\mu}$. By Theorem 1.3.2 of \cite{saleh2006} and the original result of
James and Stein \cite{james1961}, for all $p \ge 3$, 
$\mathbb{E}\|\hat{\boldsymbol{\mu}}^{S+}-\boldsymbol{\mu}\|^2<\mathbb{E}\|\mathbf{Y}-\boldsymbol{\mu}\|^2$. Undoing the translation by $\mathbf{m}_{t-1}$ yields $R_t(\hat{\mathbf{g}}_t^{S+})<R_t(\mathbf{g}_t)$, almost surely, which completes the proof.
\end{proof}
We need the following moment bounds of the Stein gradient estimator. 

\begin{lemma}\label{lem:second_moment}
Under the Gaussian noise assumption, $\mathbb{E}\|\hat{\mathbf{g}}_t^{S+}\|^2<\infty$.
\end{lemma}
\begin{proof}
Since the shrinkage factor is bounded in $[0,1]$,
\[
\|\hat{\mathbf{g}}_t^{S+}\|
\le
\|\mathbf{g}_t\| + \|\mathbf{m}_{t-1}\|.
\]

Taking expectations and using finite second moments of
$\mathbf{g}_t$ completes the proof.
\end{proof}

\begin{proof}[Proof of Theorem~\ref{thm:variance_consistency}]
Adam updates have the form
\begin{equation*}
m_{t,j}=(1-\beta_1)\sum_{k=1}^t \beta_1^{t-k} g_{k,j},\quad v_{t,j}=(1-\beta_2)\sum_{k=1}^t \beta_2^{t-k} g_{k,j}^2.    
\end{equation*}
These are geometrically weighted averages.
By strict stationarity, ergodicity, and finite fourth moments,
the ergodic theorem for weighted sums implies (see \cite{krengel1985ergodic})
\begin{equation*}
m_{t,j} \xrightarrow{L^1} \mathbb{E}[g_{t,j}], \quad v_{t,j} \xrightarrow{L^1} \mathbb{E}[g_{t,j}^2].    
\end{equation*}
Thus, $v_{t,j}-m_{t,j}^2\xrightarrow{L^1}\mathrm{Var}(g_{t,j})$. Averaging over $j=1,\dots,p$ preserves convergence by linearity.
\end{proof}

\begin{proof}[Proof of Theorem~\ref{thm:convergence}]
The parameter update is $\boldsymbol{\theta}_{t+1}=\boldsymbol{\theta}_t-\alpha_t\hat{\mathbf{g}}_t^{S+}$. Decompose
\begin{equation*}
\hat{\mathbf{g}}_t^{S+}=\nabla J(\boldsymbol{\theta}_t)+\boldsymbol{\xi}_t,    
\end{equation*}
where $(\boldsymbol{\xi}_t)$ is a martingale difference sequence with
bounded second moments (Lemma~\ref{lem:second_moment}). Thus the recursion is a stochastic approximation of Robbins--Monro type. Since $J$ is smooth and bounded below, standard results \cite{borkar2008,bottou2018} imply
\begin{equation*}
\sum_t \alpha_t \|\nabla J(\boldsymbol{\theta}_t)\|^2 < \infty
\quad \text{a.s.}.    
\end{equation*}
Hence $\liminf_{t \to \infty} \|\nabla J(\boldsymbol{\theta}_t)\| = 0$ almost surely.
\end{proof}


\begin{proof}[Proof of Theorem~\ref{thm:minimax}]
\textbf{(i) Proof of the Minimax Lower Bound}: For any estimator $\hat{\boldsymbol{\mu}}$ and any prior $\Pi$ on
$\boldsymbol{\mu}$, the minimax risk satisfies $R^*\geq\inf_{\hat{\boldsymbol{\mu}}}\mathbb{E}_{\Pi}\mathbb{E}_{\boldsymbol{\mu}}\|\hat{\boldsymbol{\mu}} - \boldsymbol{\mu}\|^2$.

Choose the Gaussian prior $\boldsymbol{\mu} \sim \mathcal{N}(\mathbf{0}, \tau^2 I_p)$. The Bayes estimator is
\begin{equation*}
\hat{\boldsymbol{\mu}}^{\text{Bayes}}=\frac{\tau^2}{\tau^2 + \sigma^2} \mathbf{g}_t,
\end{equation*}
with Bayes risk
\begin{equation*}
R_{\Pi}=p \frac{\tau^2 \sigma^2}{\tau^2 + \sigma^2}.    
\end{equation*}
Letting $\tau^2 \to \infty$ yields $R_{\Pi} \to p\sigma^2$, hence $R^* \ge p\sigma^2$.\\
\textbf{(ii) Minimaxity and Inadmissibility of the Unrestricted Gradient}:

For the estimator $\hat{\boldsymbol{\mu}} = \mathbf{g}_t$, $\mathbb{E}_{\boldsymbol{\mu}}
\|\mathbf{g}_t-\boldsymbol{\mu}\|^2=\mathbb{E}\|\boldsymbol{\varepsilon}_t\|^2=p\sigma^2$, independent of $\boldsymbol{\mu}$. Hence $\mathbf{g}_t$ attains the minimax bound and is minimax.
However, by Theorem~\ref{thm:stein_risk}, it is uniformly dominated by
$\hat{\mathbf{g}}_t^{S+}$ and therefore inadmissible.\\
\textbf{(iii) Minimaxity of the Positive-Part Stein Estimator}:
By Theorem~\ref{thm:stein_risk},
\begin{equation*}
\mathbb{E}_{\boldsymbol{\mu}}\|\hat{\mathbf{g}}_t^{S+} - \boldsymbol{\mu}\|^2 \leq
p\sigma^2 \quad \forall \boldsymbol{\mu}.    
\end{equation*}
Thus the maximum risk of $\hat{\mathbf{g}}_t^{S+}$ does not exceed the minimax lower bound. Consequently, it is minimax. Using \cite{baranchik1970}, the strict inequality for $\boldsymbol{\mu} \neq \mathbf{0}$ follows from the strict dominance property of the positive-part James--Stein estimator.

\end{proof}

\vskip 0.2in
\bibliography{sample}

@article{baranchik1970,
  author  = {Baranchik, A. J.},
  title   = {A Family of Minimax Estimators of the Mean of a Multivariate Normal Distribution},
  journal = {Annals of Mathematical Statistics},
  volume  = {41},
  number  = {2},
  pages   = {642--645},
  year    = {1970}
}

@book{borkar2008,
  author    = {Borkar, Vivek S.},
  title     = {Stochastic Approximation: A Dynamical Systems Viewpoint},
  publisher = {Cambridge University Press},
  address   = {Cambridge},
  year      = {2008}
}

@article{bottou2018,
  author  = {Bottou, L{\'e}on and Curtis, Frank E. and Nocedal, Jorge},
  title   = {Optimization methods for large-scale machine learning},
  journal = {SIAM Review},
  volume  = {60},
  number  = {2},
  pages   = {223--311},
  year    = {2018}
}

@article{duchi2011,
  author  = {Duchi, John and Hazan, Elad and Singer, Yoram},
  title   = {Adaptive subgradient methods for online learning and stochastic optimization},
  journal = {Journal of Machine Learning Research},
  volume  = {12},
  number  = {61},
  pages   = {2121--2159},
  year    = {2011}
}

@book{efron2012,
  author    = {Efron, Bradley},
  title     = {Large-Scale Inference: Empirical {B}ayes Methods for Estimation, Testing, and Prediction},
  series    = {Institute of Mathematical Statistics Monographs},
  volume    = {1},
  publisher = {Cambridge University Press},
  year      = {2012}
}

@inproceedings{james1961,
  author    = {James, W. and Stein, C.},
  title     = {Estimation with quadratic loss},
  booktitle = {Proceedings of the Fourth Berkeley Symposium on Mathematical Statistics and Probability},
  volume    = {1},
  pages     = {361--379},
  year      = {1961}
}

@inproceedings{kingma2015,
  author    = {Kingma, Diederik P. and Ba, Jimmy},
  title     = {Adam: A method for stochastic optimization},
  booktitle = {International Conference on Learning Representations (ICLR)},
  year      = {2015}
}

@book{krengel1985ergodic,
  author    = {Krengel, Ulrich},
  title     = {Ergodic Theorems},
  series    = {de Gruyter Studies in Mathematics},
  volume    = {6},
  publisher = {Walter de Gruyter \& Co.},
  address   = {Berlin},
  year      = {1985}
}

@article{mandt2017,
  author  = {Mandt, Stephan and Hoffman, Matthew D. and Blei, David M.},
  title   = {Stochastic Gradient Descent as Approximate {B}ayesian Inference},
  journal = {Journal of Machine Learning Research},
  volume  = {18},
  number  = {134},
  pages   = {1--35},
  year    = {2017}
}

@article{polyak1964,
  author  = {Polyak, Boris T.},
  title   = {Some methods of speeding up the convergence of iteration methods},
  journal = {USSR Computational Mathematics and Mathematical Physics},
  volume  = {4},
  number  = {1},
  pages   = {1--17},
  year    = {1964}
}

@book{saleh2006,
  author    = {Saleh, A. K. Md. Ehsanes},
  title     = {Theory of Preliminary Test and {S}tein-Type Estimation with Applications},
  publisher = {Wiley},
  year      = {2006}
}

@inproceedings{stein1956,
  author    = {Stein, C.},
  title     = {Inadmissibility of the usual estimator for the mean of a multivariate normal distribution},
  booktitle = {Proceedings of the Third Berkeley Symposium on Mathematical Statistics and Probability},
  volume    = {1},
  pages     = {197--206},
  year      = {1956}
}

@inproceedings{sutskever2013,
  author    = {Sutskever, Ilya and Martens, James and Dahl, George and Hinton, Geoffrey},
  title     = {On the importance of initialization and momentum in deep learning},
  booktitle = {Proceedings of the 30th International Conference on Machine Learning (ICML)},
  pages     = {1139--1147},
  year      = {2013}
}

@article{tieleman2012,
  author  = {Tieleman, Tijmen and Hinton, Geoffrey},
  title   = {Lecture 6.5-rmsprop: Divide the gradient by a running average of its recent magnitude},
  journal = {COURSERA: Neural Networks for Machine Learning},
  volume  = {4},
  number  = {2},
  pages   = {26--31},
  year    = {2012}
}

\end{document}